\documentclass{article}

    \PassOptionsToPackage{round, sectionbib}{natbib}

    \usepackage[preprint]{neurips_2019}

\usepackage[utf8]{inputenc} \usepackage[T1]{fontenc}    \usepackage[hidelinks]{hyperref}       \usepackage{url}            \usepackage{booktabs}       \usepackage{amsfonts}       \usepackage{nicefrac}       \usepackage{microtype}      
\usepackage{amsmath, amsthm}
\usepackage{graphicx}
\usepackage{wrapfig}
\usepackage{subcaption}

\graphicspath{{./figures/}}

\newcommand{\II}{\mathcal{I}}
\newcommand{\RR}{\mathbb{R}}
\newcommand{\Ss}{\mathcal{S}}
\newcommand{\Aa}{\mathcal{A}}
\newcommand{\Oo}{\mathcal{O}}
\newcommand{\EE}{\mathcal{E}}

\newcommand{\thetaold}{{\theta_\textnormal{old}}}

\newtheorem{lemma}{Lemma}
\newtheorem*{remark}{Remark}

\DeclareMathOperator*{\argmax}{arg\,max}

\title{Attentional Policies for Cross-Context Multi-Agent Reinforcement Learning}

\author{  Matthew A.~Wright\thanks{Corresponding author.}~~and~Roberto Horowitz
  \\
    University of California\\
  Berkeley, CA, USA \\
  \texttt{\{mwright,horowitz\}@berkeley.edu} \\
                                          }

\begin{document}

\maketitle

\begin{abstract}
Many potential applications of reinforcement learning in the real world involve interacting with other agents whose numbers vary over time.
We propose new neural policy architectures for these multi-agent problems.
In contrast to other methods of training an individual, discrete policy for each agent and then enforcing cooperation through some additional inter-policy mechanism, we follow the spirit of recent work on the power of relational inductive biases in deep networks by learning multi-agent relationships at the policy level via an attentional architecture.
In our method, all agents share the same policy, but independently apply it in their own context to aggregate the other agents' state information when selecting their next action.
The structure of our architectures allow them to be applied on environments with varying numbers of agents.
We demonstrate our architecture on a benchmark multi-agent autonomous vehicle coordination problem, obtaining superior results to a full-knowledge, fully-centralized reference solution, and significantly outperforming it when scaling to large numbers of agents.
\end{abstract}

\section{Introduction}
Deep reinforcement learning (RL) has been at the core of many breakthroughs in AI and controls domains in recent years.
Examples include robotic locomotion \citep{lillicrap_continuous_2016} and strategy games like Go at which computers had previously been noncompetitive \citep{silver_mastering_2016,silver_mastering_2017}.

The graduation of deep RL systems from closed environments to the real world will introduce new complexities.
In the real world, a system will need to learn to interact not just with the environment, but with other (naturally or artificially) intelligent agents as well: they are \emph{multi-agent systems} \citep{stone_multiagent_1997,hernandez-leal_is_2018}.

Many new algorithms for deep RL have come from the statement of desired behaviors in policy learning, and then constructing new RL objective functions to encourage those desiridata.
Such a pattern has been apparent in the multi-agent setting as well.
Several recent papers on multi-agent-specific RL training algorithms propose value function estimators that can relate multiple agents' states and actions to observed scalar returns.
For example, \citet{sunehag_value-decomposition_2017} and \citet{rashid_qmix_2018} consider how to decompose a joint action-value function (Q function) into per-agent ones that, after learning, can guide each agent independently.
Other works (\citet{foerster_counterfactual_2017,lowe_multi-agent_2017,mao_modelling_2019}, etc.) consider a joint value function as playing the role of a critic in a multi-agent generalization of RL's actor-critic paradigm.
The per-agent policy networks (i.e., the actors) are trained with \emph{centralized} critic networks that aggregate information from all actors, then provide learning signals to move the actors to more cooperative behaviors.

In this paper, we take a different approach.
Taking a page from recent work in recognizing the utility of \emph{relational inductive biases} \citep{battaglia_relational_2018} when crafting the interior architectures of a deep learning system, we redesign the policy networks such that agents can learn how to interact with other agents at the \emph{policy} level.
In particular, we apply neural attention \citep{bahdanau_neural_2015,vaswani_attention_2017} as a fundamental building block in our policy networks.
We argue that this framework has appealing properties: among other benefits, it allows us to apply classic ``single-agent'' RL algorithms to the multi-agent problem with few modifications.
The same architectures can be used for critic networks as well, which provides an appealing return to the classic, ``single-agent'' world, where actor and critic networks are often just architectural copies (c.f. the aforementioned methods where the critic networks are more complex than the policy networks).

Finally, the attentional construction lets us flexibly apply the same network to different contexts, both across agents and across different environmental situations from the perspective of each agent (i.e., we can both evaluate and train the same policy on multi-agent environments with different numbers of agents).
We call the architecture ``\emph{cross-context}'' to emphasize this flexibility.

The remainder of this paper is organized as follows.
In section \ref{sec:definitions}, we briefly review the mathematical framework of RL in general and multi-agent RL.
Section \ref{sec:relatedwork} reviews recent work in multi-agent RL and the use of ``relational inductive biases'' (of which attention is a particular instance) in learning.
Section \ref{sec:problem} reviews a benchmark multi-agent RL problem in coordinated autonomous vehicle control \citep{vinitsky_benchmarks_2018} that we use as a framing problem.
We then proceed to implementation: section \ref{sec:architecture} discusses our attentional policy networks for RL, section \ref{sec:appo} discusses how these networks enable straightforward application of single-agent RL methods to multi-agent problems, and section \ref{sec:implementation} gives various implementation details.
Section \ref{sec:results} presents our results, and finally, section \ref{sec:conclusion} concludes with discussion on how the attentional policies enable flexible and decentralized multi-agent RL.

\section{Definitions and Preliminaries}
\label{sec:definitions}
\subsection{The general RL setting}
\label{subsec:rl}
RL is typically presented in the mathematical framework of finite-horizon, discounted Markov decision processes (MDPs) \citep{duan_benchmarking_2016}.
These MDPs are defined by a tuple $(\Ss, \Aa, P, r, \rho_0, \gamma, T)$, where $\Ss$ is the state space, $\Aa$ is the action space,
$P$ is a stochastic transition function from $\Ss \times \Aa$ to $\Ss$ (i.e., it defines a conditional probability distribution
$P(s_{t+1} | s_t, a_t)$ for 
$s_t, s_{t+1} \in \Ss, a_t \in \Aa$), 
$r: \Ss \times \Aa \to \RR$ is the reward function, $\rho_0$ is a probability distribution on initial states $s_0$, $\gamma \in (0, 1]$ is a reward discounting factor, and $T$ is the time horizon.
The goal is to maximize the cumulative discounted reward $\sum_{t=0}^T \gamma^t r(s_t, a_t)$ where $s_t$ and $a_t$ are the state and action, respectively, at time $t$.

In the RL problem, the probability distributions and/or the reward function are unknown.
The objective is to learn a \emph{policy} $\pi$ that defines a conditional probability distribution of actions given states, $a_t \sim \pi(a_t | s_t)$, such that the policy approximately maximizes the expectation of the discounted reward.

Typically we define our space of candidate policies as parameterized by some parameter vector $\theta$ (in modern deep RL, $\theta$ is the weights in a deep neural network).
That is, the full expression for the policy is $\pi(a_t | s_t ; \theta)$ (also written with $\theta$ as a subscript, $\pi_\theta(a_t | s_t)$).
The RL objective is then to find the optimal parameter vector $\theta^*$, defined as
\begin{equation}
    \theta^* = \argmax_\theta E_\tau \left[ \sum_{t=0}^T \gamma^t r(s_t, a_t) \right]
    \label{eq:thetastar}
\end{equation}
where $\tau = (s_0, a_0, s_1, a_1, \dots)$ is a shorthand for the entire trajectory.
The policy $\pi_\theta(a_t | s_t)$ appears in \eqref{eq:thetastar} as part of the distribution over which the expectation is being taken.

Solving an RL problem requires devising a procedure for moving around in $\theta$-space to gather information about $P$ and $r$ (in the form of experienced samples) and using the information about the landscape of $P$ and $r$ to move towards $\theta^*$.
The abstract entity that draws actions $a_t$ from $\pi_\theta(a_t | s_t)$, executes them, and observes the resulting next state $s_{t+1} \sim P(s_{t+1} | s_t, a_t)$ and reward value $r(s_t | a_t)$ (for the particular $s_t,a_t$) is typically called the ``agent.''

\subsection{Multi-agent RL}
\label{sec:multiagentrl}
So far, we have just described the background to traditional, non-multi-agent RL.
Multi-agent RL, as its name suggests, adds complications by having multiple agents.
Let $\II_t$ denote the set of agents that are present in the environment at time $t$.
We will say that each agent has its own observation space $\Oo^i$, which reflects a (partial) observation of the environment from that agent's perspective,
and its own action space $\Aa^i$ and policy $\pi^i$ which defines a per-agent conditional probability distribution on actions given the state, $a_t^i \sim \pi^i(a_t^i | s_t)$, with $a_t^i \in \Aa^i$.

In general, agent $i$ need not only have information (sensor readings, etc.) \emph{physically} local to itself.
In a multi-agent setting where agents are meant to interact, it reasonable to assume that each agent will have information about the others.
This information may be obtained by agent $i$ from, for example, inter-agent communication or visual observation of other agents.

In the particular implementation presented in this work, the information $i$ receives from $j$, if any, is $o_t^j$, i.e., $j$'s observation.
A more general case (e.g., where $i$'s information about itself and about other agents are of different dimensionalities) is possible, but in this paper we will say that the local and received observations live in the same space to simplify later notation.

We also define a directed graph with edge set $\EE$ that encodes the inter-agent relationships.
The edge $(i,j) \in \EE$ if $i$ gets information from $j$.
Each edge also has a particular class $c$, $c \in \{1,\dots,C\}$, to encode that agents relate to each other in meaningfully different ways.

\section{Related Work}
\label{sec:relatedwork}
Multi-agent RL is said to be much more difficult than conventional, single-agent, RL.
In addition to the typical obstacles in single-agent RL (like temporal credit assignment due to sparse rewards and navigating the exploration-exploitation tradeoff), multi-agent RL adds complications such as an intrinsically higher dimensionality, per-agent credit assignment, and (from the perspective of each individual agent) environmental nonstationarity during the learning process (i.e., if multiple interacting agents are all learning at the same time, then one agent's knowledge about how others react to their actions quickly becomes outdated) \citep{hernandez-leal_is_2018}.

\subsection{Learning Multi-agent Cooperation}
Several authors have proposed new RL training regimes to encourage the learning of cooperative behavior.
In general, these approaches retain the idea of training individual policies per agent, but adjust the training goal to include context-specific multi-agent information.

\citet{sunehag_value-decomposition_2017} and \citet{rashid_qmix_2018} consider the multi-agent Q-learning problem, where the joint (global) action-value function $Q(s_t, a_t)$ is well-defined, and study structural decompositions of it into per-agent action-value functions $Q^i(s_t^i, a_t^i)$ such that each agent's optimum $a_t^i$ should be the same as the action it would have taken if the joint Q-function was maximized.

\citet{foerster_counterfactual_2017,lowe_multi-agent_2017,iqbal_actor-attention-critic_2018,mao_modelling_2019}, among others, consider multi-agent RL where individual agents operate independently, but during actor-critic-style training, a centralized critic takes in information from all other agents.
During policy execution (where the critic is not present), each agent operates independently, but has learned the behaviors favored by the centralized critic.
\citet{iqbal_actor-attention-critic_2018} and \citet{mao_modelling_2019}, in particular, propose centralized critics whose value function estimates make use of an attention module.
Our proposed method is somewhat similar to the centralized critic methods, in that we also make use of a value function baseline that aggregates information (although our implementation is different from the aforementioned references), but we also permit the agents to explicitly attend to each other during execution.
Of particular importance is that this means that the policies can adapt to situations of \emph{varying numbers} of other agents.

Such a distinction is similar to the distinction between ``self-attention,'' popularized by \citet{vaswani_attention_2017} and the traditional attention of \citet{bahdanau_neural_2015} (sometimes called ``encoder-decoder'' or just ``decoder'' attention).
Our policy network structures can be thought of implementing self-attention,\footnote[1]{The name is somewhat confusing since in our ``self''-attention, agents are in fact attending to other agents. ``Self-attention'' is better thought of as being distinct from encoder-decoder attention in that in self-attention, all entities attend on each other, whereas in encoder-decoder attention, a fixed decoder attends on multiple encoded attendees, with the attentional decoding only happening in one direction.} whereas a centralized critic uses the classic encoder-decoder attention.
To the best of our knowledge, this work represents the first proposal of attention mechanisms in the policy network itself.

\subsection{Multi-agent Communication}
In this work, we take the information that agent $i$ has about agent $j$ as given.
The receiving/observing agent $i$ then learns to process this information when generating its policy distribution.
This is in contrast to a body of work (\citet{foerster_learning_2016,sukhbaatar_learning_2016,lazaridou_multi-agent_2017,mordatch_emergence_2018}, etc.) in which a communication policy is explicitly learned.
In an application where $i$'s information about $j$ is \emph{communicated} by $j$ rather than \emph{observed} by $i$, and communication has some associated bandwidth constraint or cost, it makes sense to add a communication policy in the learning objective.

As an example, \citet{jiang_learning_2018} proposed an attentional module to enable agents to decide \emph{when} to communicate, and following \citet{sukhbaatar_learning_2016} and \citet{peng_multiagent_2017} a global LSTM coordinator is used to aggregate and disseminate this information back to the agents.
They argue that enabling the learning of \emph{selective} communication improves performance, by removing the need for receivers to filter out less-useful information.
An extension of our framework where the attended-on information is emitted by a learned communication module as in \cite{jiang_learning_2018}, but the processing is done in a decentralized agent-wise manner like in this work, is an interesting avenue for future work.

\subsection{Relational Inductive Biases in Learning}

Recent work in machine learning theory (\citet{battaglia_relational_2018} present a broad review) has argued for the importance of so-called \emph{Relational Inductive Biases} in effective learning.
A relational inductive bias is an inductive bias (defined by \citet{battaglia_relational_2018}, citing \citet{mitchell_need_1980}, as a bias or prior that ``allows a learning algorithm to prioritize one solution (or interpretation) over another, independent of the observed data'') that encodes prior knowledge about the existence of discrete entities, and the \emph{relations} among those entities.
A convolutional neural network layer, for example, encodes a relational inductive bias that pixels are the entities of interest, and that \emph{locality} of statistically correlated pixels is of prime importance \citep{battaglia_relational_2018}.

Neural attention \citep{bahdanau_neural_2015,luong_effective_2015,vaswani_attention_2017} has been characterized as a form of relational inductive bias, where the attended-on elements are the entities of interest, and the relations are quantified via the ``attention weights,'' the output of a learned attention module \citep{battaglia_relational_2018}.
Giving each agent a relational inductive prior to the other agents it needs to coordinate with, at the policy level, motivates our application.

\section{Our Framing Problem: Open-Network Autonomous Vehicle Coordination}
\label{sec:problem}
In this work, we consider a benchmark multi-agent RL problem introduced by \cite{vinitsky_benchmarks_2018} (shown in Figure \ref{fig:merge}).
The work proposed several multi-agent reinforcement learning problems based on mixed-autonomy traffic (road traffic with mixtures of autonomous and human-driven vehicles).
We will consider the ``Merge'' problem.
In this problem, two single-lane roads merge into one.
At the merge, the vehicles will compete for space, inducing congestion and a high social cost.
The RL problem is to take control of some subset of the vehicles and dissipate this congestion.

The non-controlled vehicles are modeled as being driven by humans, and their accelerations are given by a behavioral model called the \emph{Intelligent Driver Model} \citep{treiber_congested_2000}.

\begin{wrapfigure}[10]{R}{.6\textwidth}
    \centering
    \includegraphics[width=.6\textwidth]{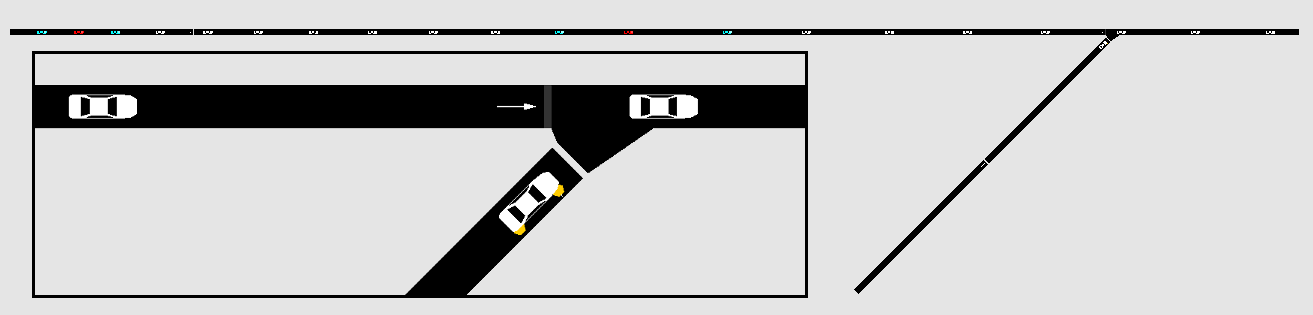}
    \caption{``Merge'' benchmark road network, with zoom-in showing simulated merging vehicles.}
    \label{fig:merge}
\end{wrapfigure}

In \cite{vinitsky_benchmarks_2018}, the canonical solution uses a centralized single-agent approach rather than a multi-agent approach.
There, a central controller receives all observations, stacks them into one vector, and computes all actions.
However, the number of controlled vehicles on the network will change as they enter and exit, so to use a traditional single-agent MLP (successive fully-connected neural network layers) architecture, a fixed maximum number $N$ of vehicles to control must be defined, and the network-wide observation vector is either truncated or zero-padded as needed.
On the action end, if $|\II_t| < N$, extra actions are discarded, and when $|\II_t| > N$, some are left uncontrolled.

\label{sec:why}

This style of centralized controller has some drawbacks beyond its potential lack of realism.
Controlling at most $N$ vehicles effectively throws away extra information when more than $N$ vehicles are present.
The padding and truncation likely also makes the learning problem harder because the RL agent is expected to learn by itself to not assign credit to the ignored actions (without knowledge that they have been ignored), making the credit assignment problem even more difficult.
One solution is to train different policies for different numbers of agents and select between them as the situation changes, but training many policies would, among other issues, vastly increase the RL sample requirements.
In contrast, our proposed method seeks to be cross-trainable by allowing valid backpropagation for any number of agents.

Autonomous vehicle coordinative algorithms like Cooperative Adaptive Cruise Control methods call for decentralized controllers \citep{dey_review_2016,wang_review_2018}.
A method to decentralize the training and execution of RL controllers is a critical step towards their deployment to real transportation networks.
Our method has an advantage in this area in that, since it is valid for any number of agents, it can by construction be executed by a single agent in a fully decentralized manner.

\section{Attentional Architectures for RL}
\label{sec:architecture}

We use a (self-) attention layer in our policy and value networks.
The layer takes in a set of input vectors $\{h_{in}^i \in \RR^n, i \in \II_t\}$ and outputs a set of vectors $\{h_{out}^i \in \RR^m,  i \in \II_t\}$.
The layer calculation is the scaled-dot-product attention calculation \citep{vaswani_attention_2017} with a modified version of \citet{shaw_self-attention_2018}'s ``relative positional embeddings'' to differentiate different types of edges, which in our notation is
\begin{subequations}
    \label{eq:attn}
    \begin{equation}
        h^i_{out} = \sum_{j \in \II_t: \; (i,j) \in \EE} \alpha^{ij} \left( h_{in}^j W_V + a_V^{c(i,j)} \right) \label{eq:attn1}
    \end{equation}
where $c(i,j)$ means the class of the edge $(i,j)$, $\alpha^{ij}$ is the attention weight of $i$ attending on $j$, given by
    \begin{equation}
        \alpha^{ij} = \frac{\exp(e^{ij})}{\sum_{k \in \II_t: \; (i,k) \in \EE} \exp(e^{ij})} 
        \qquad \quad e^{ij} = \frac{h_{in}^i W_Q (h_{in}^j W_K + a_K^{c(i,j)})^T}{\sqrt{m}} \label{eq:attn2}
    \end{equation}
and where $W_Q, W_K, W_V \in \RR^{n \times m}$ and $a_V^{c(i,j)}, a_K^{c(i,j)} \in \RR^m$ for $c \in \{1, \dots, C\}$ are learned parameters.
\end{subequations}
The calculations in \eqref{eq:attn} form a straightforward self-attention calculation in the spirit of \citet{vaswani_attention_2017}, with the adoption of \citet{shaw_self-attention_2018}'s proposal to have unique bias vectors $a_V$ and $a_K$ for each of the $C$ different ways that agents can relate to each other.
We also make use of \emph{multi-head} attention \citep{vaswani_attention_2017}, where the calculations in \eqref{eq:attn} are performed multiple times in parallel, with independent $W$'s and $a$'s, and then each head's output vectors are concatenated.

In this work, our policy network is as follows.
The per-agent observations $o^i_t \in \RR^n$, $i \in \II_t$ are stacked into an $|\II_t| \times n$ tensor.
This tensor goes into the attention layer described in \eqref{eq:attn}.
We use 4 attention heads of $m=16$ units each.
Next is a fully-connected hidden layer with 64 units (each of the $|\II_t|$ attention layer outputs pass through this layer identically and in parallel).
Both the attentional and fully-connected sublayers are followed by a ReLu nonlinearity and a layer-normalization operation \citep{ba_layer_2016} (with learned scale and location parameters), in that order.
This structure of an attentional sublayer followed by a shared-over-agents fully-connected sublayer is inspired by the Transformer architecture of \cite{vaswani_attention_2017}, though we use only one such layer and omit residual connections.

The output of the above layers then goes into the output layer, whose output parameterizes the stochastic policy.
In this work, our stochastic policy is a per-agent Gaussian distribution with mean and log-variance computed by the same fully-connected layer for each agent.
To reiterate, the same layers are used for all vehicles $i \in \II_t$, and can be computed fully in parallel.
Batching multiple together can be done by padding, and ensuring the pad vehicles do not contribute to the sum in \eqref{eq:attn1}.

\section{Attentional Multi-Agent Proximal Policy Optimization}
\label{sec:appo}
One key boon of using self-attentional architectures in the policy network is that classic ``single-agent'' RL training algorithms can be used with little modification (this is in contrast to, e.g., \citet{sunehag_value-decomposition_2017,foerster_counterfactual_2017,lowe_multi-agent_2017,iqbal_actor-attention-critic_2018,rashid_qmix_2018,mao_modelling_2019}, where restricting learning about inter-agent relationships to \emph{outside} of the policy network requires multi-agent-specific modifications to RL).
This section outlines how one can use Proximal Policy Optimization (PPO) \citep{schulman_proximal_2017}, a popular and relatively simple RL algorithm, to train our attentional multi-agent policy.

\paragraph{Vanilla Proximal Policy Optimization}

A general PPO objective function of a parameter vector $\theta$ at timestep $t$ is of the form \citep{schulman_proximal_2017}
\begin{subequations}
    \label{eq:ppo}
\begin{equation}
    L^{PPO}_t(\theta) = E_\pi \left[L^{CLIP}_t(\theta) - c_1 L^{VF}_t(\theta) + c_2 S\left(\pi_\theta \right) - \beta D_{KL}(\pi_\thetaold || \pi_\theta ) \right] \label{eq:ppo_pt1}
\end{equation}
with the ``clipped'' surrogate advantage objective function
\begin{equation}
    L^{CLIP}_t(\theta) = \min \left( 
        \frac{\pi_\theta}{\pi_\thetaold} \cdot A_t, \;
        \textnormal{clip}
            \left( \frac{\pi_\theta}{\pi_\thetaold},
            1-\epsilon,
            1 + \epsilon \right) \cdot A_t \right) \label{eq:ppo_pt2}
\end{equation}
\end{subequations}
where $\pi_\thetaold$ is the policy distribution under some fixed parameter vector $\thetaold$ (taken to be the distribution at the beginning of a PPO iteration), and $A_t$ is the advantage at time $t$.
The policy $\pi_\theta$ is still a conditional probability distribution for the random variable $a_t | s_t$, but we omit this argument in \eqref{eq:ppo} to reduce notational clutter.
Also, $L_t^{VF}$ is the squared error of an estimate of the value function $V(s_t)$, $S(\pi_\theta)$ the entropy of the distribution $\pi_\theta$, $D_{KL}(\cdot || \cdot )$ the KL divergence, $c_1$ and $c_2$ are scaling constants, and $\beta$ is a scalar whose value is updated adaptively during the training process.

The formulation of the PPO objective is to encourage the distribution $\pi_\theta(a_t | s_t)$ to move towards higher-advantage actions, but not move $\pi_\theta(a_t | s_t)$ \emph{too far} from $\pi_\thetaold(a_t | s_t)$.
Some implementations of PPO do not use the KL penalty ($\beta = 0$), some do not use the clipping of the objective ($\epsilon \to \infty$), and some use both to enforce this ``distance'' bound.

In practice, the expectation in \eqref{eq:ppo_pt1} is approximated via a sample mean of $(s_t, a_t, r_t)$ tuples.
That is, every term in \eqref{eq:ppo} is evaluated pointwise for a batch of samples and the gradient with respect to $\theta$ is computed against the mean of the pointwise $L_t^{PPO}(s_t, a_t, r_t)$.

\paragraph{The Distribution for the Attentional PPO Objective}

The traditional PPO algorithm assumes only a single agent.
How do we apply the objective $\eqref{eq:ppo}$ when the domains of $s_t$ and $a_t$ vary over timesteps with the number of agents?

The key insight that allows us to apply PPO to the attentional multi-agent architecture is that every term is a function of statistics of the \emph{policy} $\pi_\theta$ rather than any particular \emph{agent}.
The form of the joint distribution $\pi_\theta$ across agents varies over timesteps, but we reduce the information \emph{across agents} into per-timestep scalar statistics of $\pi_\theta$ (entropy and KL divergence in \eqref{eq:ppo_pt1}, likelihood ratios in \eqref{eq:ppo_pt2}) before actually taking the expectation (i.e., before averaging over timesteps).

We formalize this result in the following elementary lemma:
\begin{lemma}
    For the attentional policy network, the conditional distribution of the actions given the state for timestep $t$ is given by
    \begin{equation}
        \pi_\theta(a_t | s_t) = \prod_{i \in \II_t} \pi_\theta(a_t^i | \{ o_t^j: (i,j) \in \EE \}).
    \end{equation}    
\end{lemma}
\begin{proof}
    We have defined that $\Aa^i$ is the action space for agent $i$, and defined the joint action space as the product space $\Aa = \prod_i \Aa^i$.
    Then, we can define the product measure on $\Aa$ as simply the product of the measures on the component action spaces,
    \begin{equation}
        \pi_\theta(a_t | s_t) = \prod_{i \in \II_t} \pi_\theta(a_t^i | s_t). \label{eq:proof1}
    \end{equation}
    Then, we note that by the construction of the attentional network, we have that $a_t^i$ is conditionally independent of $\{o_t^j: (i,j) \notin \EE\}$ given $\{o_t^j: (i,j) \in \EE\}$,
\begin{equation}
    \pi_\theta(a_t^i | s_t) = \pi_\theta(a_t^i | \{ o_t^j: (i,j) \in \EE \}). \label{eq:proof2}
\end{equation}
    Combining \eqref{eq:proof1} and \eqref{eq:proof2}, we immediately have the lemma.
\end{proof}
\begin{remark}
    Note that the above construction only makes sense because $\theta$ and $s_t$ are held fixed for all $i$.
    This means that the $a_t^i$ are exchangeable in the de Finetti sense.
    In a non-attentional network where the relational inductive bias does not encode a permutation invariance (e.g., if an LSTM is used to sequentially encode the observations $i$ attends to), this may not hold.
\end{remark}

\paragraph{Attentional Value Function Baseline}
In this work, we estimate the value function $V(s_t)$ using a neural network with identical architecture to the policy network described above, but with an agent-wise max pooling operation at the end whose output goes through a fully-connected layer to produce a scalar value.
This can be seen as a fully self-attentional critic rather than the encoder-decoder attentional critics used in, e.g., \citet{iqbal_actor-attention-critic_2018,mao_modelling_2019}.
This value function estimator is used in the Generalized Advantage Estimator \citep{schulman_high-dimensional_2016} to estimate $A_t$ in \eqref{eq:ppo_pt1}.

\section{Implementation Details}
\label{sec:implementation}
The ``merge'' baseline described above is implemented in the framework \emph{Flow} \citep{wu_flow_2017}, which is a Python codebase built on the widely-used microscopic vehicle traffic simulator SUMO \citep{krajzewicz_recent_2012} that adapts SUMO to the widely-used RL problem standard ``\texttt{env}'' developed in OpenAI \emph{Gym} \citep{brockman_openai_2016}.
We implemented our neural network architecture in \emph{Ray} \citep{moritz_ray_2018}, specifically its RLlib framework \citep{liang_rllib_2018}.
In particular, we modified RLlib's implementation of PPO to be compatible with the network architecture we described above.
All PPO hyperparameters were left as the same as in \citet{vinitsky_benchmarks_2018}, with the exception that we update our policy every 20 rollouts instead of every 50.

We also used Ray to produce a baseline solution similar to \citet{vinitsky_benchmarks_2018}'s fully-centralized single-agent approach, using MLP policy and value networks with the padding and truncation discussed in section \ref{sec:why}.
For our single-agent reference, we use a two-hidden-layer networks with 64 units in each fully-connected layer and a tanh nonlinearity in between.
This 64x64 architecture serves as a comparison to the attentional architecture that has the same number of hidden units.

\section{Experimental Results}
\label{sec:results}

\cite{vinitsky_benchmarks_2018} proposed several different configurations of the ``Merge'' problem, varying in the penetration rate of autonomous vehicles and the maximum number of vehicles that are allowed to be controlled.
At the low end, ``Merge 0'' requires the control of at most 5 vehicles, and on the high end, ``Merge 2'' requires the control of up to 17 vehicles.
We report the results of several experiments in Figure \ref{fig:results} and discuss them in detail below.

\begin{figure}[t]
    \centering
    \begin{subfigure}{.32\textwidth}
        \includegraphics[width=\textwidth]{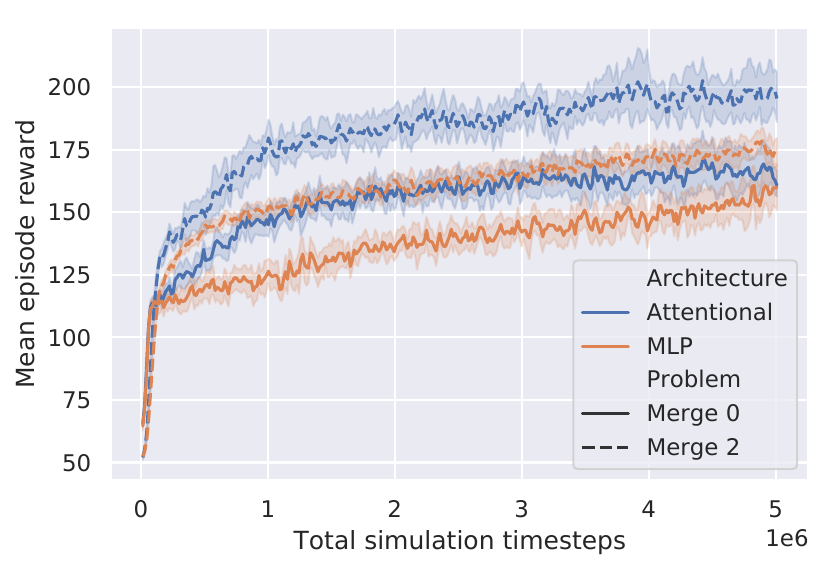}
        \caption{}
        \label{fig:base}
    \end{subfigure}
    ~
    \begin{subfigure}{.32\textwidth}
        \includegraphics[width=\textwidth]{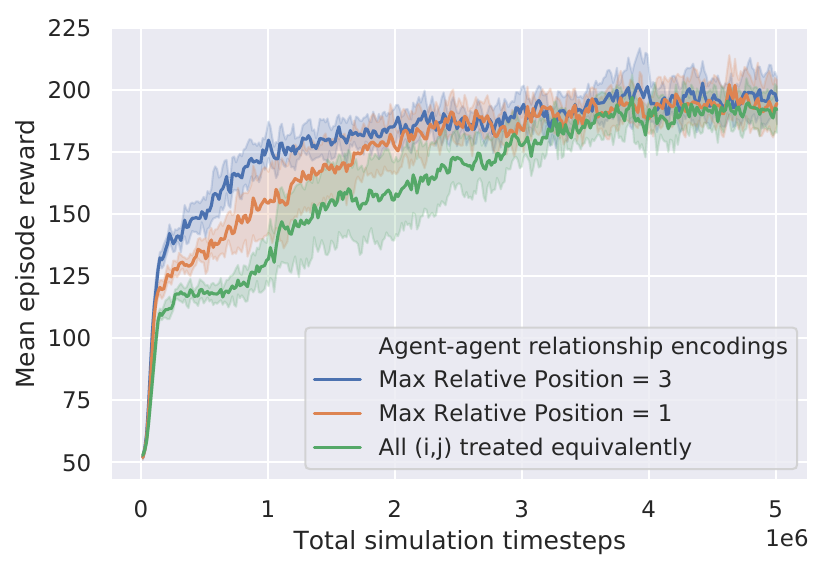}
        \caption{}
        \label{fig:relpos}
    \end{subfigure}
    ~
    \begin{subfigure}{.32\textwidth}
        \includegraphics[width=\textwidth]{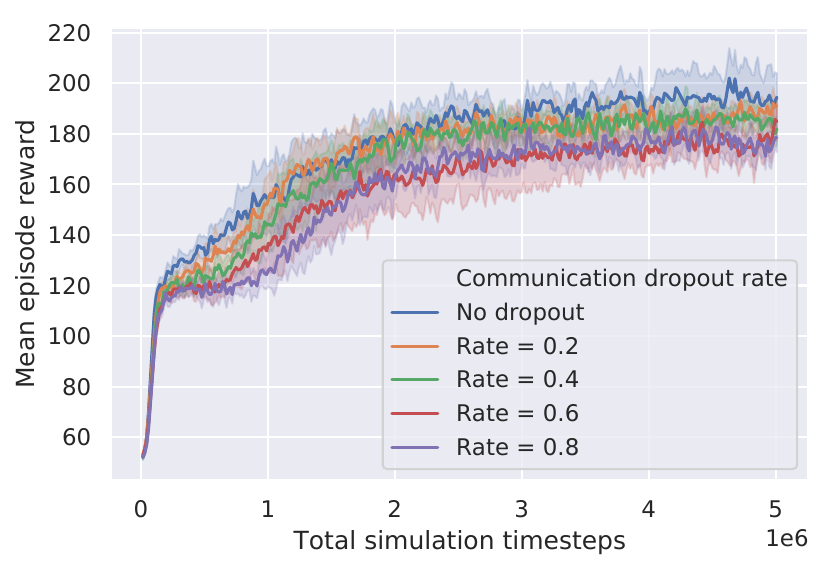}
        \caption{}
        \label{fig:dropout}
    \end{subfigure}
    \caption{Learning curves for various experiments. All curves show the mean and 95\% confidence interval of mean episode reward.
    Figure (a) compares the performance of the attentional policy to a standard MLP policy with dynamic padding and truncation to a fixed size.
    Figure (b) studies the performance for different levels of explicitly-encoded relational inductive bias.
    Figure (c) examines the robustness of the attentional policies to lossy communication by varying the degree to which communicated observations are randomly dropped.
    All curves are for sample sizes of 10 runs with different random seeds.}
    \label{fig:results}
\end{figure}

Figure \ref{fig:results}(a) shows learning curves for PPO on the ``Merge 0'' and ``Merge 2'' benchmarks, for both our attentional architecture and the reference MLP architecture with padding and truncation described in section \ref{sec:problem}.
On both problems, we obtain superior performance to the MLP architecture.

\paragraph{Importance of Relational Inductive Biases} Some of the superior performance of the decentralized controller comes from the power of the relational inductive biases encoded in the attention module.
To study this, in Figure \ref{fig:results}(b) we experiment on Merge 2 with varying numbers of relative position encodings $C$.
The base case uses $C=7$, giving unique $c(i,j)$'s for the self-case of $i=j$ and the two subsequent controlled vehicles upstream and downstream; all other further-upstream vehicles share the same $c(i,j)$ relation, and all other further-downstream another $c(i,j)$.
This configuration is ``Max Relative Position~=~3'' in figure \ref{fig:results}(b).
The line labeled ``Max Relative Position~=~1'' uses only $C=3$, where all downstream and all downstream vehicles to the attending agent are considered equivalently.
Finally, the line labeled ``All (i,j) treated equivalently'' means that $C=1$, and each agent treats both itself and all other agents equivalently.
We see that all configurations are able to eventually attain around the same maximum reward (the fact that we use multi-headed attention means that even for the all-agents-equivalent case, the policy can learn to use different heads to attend to different $o^j_t$'s in different ways), but more informative relational inductive biases give increases in sample efficiency and less variance in learning.

\paragraph{Robustness to Varying Information Availability} Finally, in figure \ref{fig:results}(c) we test the attentional policies' robustness by introducing randomly lossy communication.
On Merge 2 with the ``Max Relative Position = 1'' configuration, on every timestep we randomly delete edges $(i,j)$ for $i\neq j$ with varying probabilities.
This means that each agent $i$ will randomly not be able to attend to information from other agents.
We find that while there may be somewhat of a performance decrease as the dropout rate increases, it is minor (the difference between the end-of-training mean reward for the 0.8-rate and both the 0.2-rate and no-dropout cases is statistically significant under a two-sample unequal variances $t$-test ($p \approx 0.01$ for both), but none of the other pairwise differences are).
This suggests that even for highly-varying information environments, the architecture can generalize.

\section{Conclusion: Attention's Real-World Applicability}
\label{sec:conclusion}
It is worth noting a few details that make the attentional architecture appealing for multi-agent RL problems.
Of key importance is that each agent's actions are computed fully in parallel.
This means that each agent can actually compute its action locally, independent of the other agents, using only its knowledge of its and whatever other agents' states it has available. 
Computing all agents' actions in batch is only for purposes of computational parallelism and ease of explanation.

Also of note is how using the attentional architecture allows for the straightforward application of a simple and relatively well-understood \emph{single-agent} RL training algorithm (namely, PPO).
The question of how each agent needs to reason about all other agents when determining its own action is explicitly moved to the policy network.
The ability to deploy classic RL algorithms like PPO, as opposed to needing multi-agent-specific RL algorithms like QMIX \citep{rashid_qmix_2018} is noteworthy.

Since all agents use the same policy, we may think about each agent's state and action, and its view of the states of the other agents, as an individual training example for the single policy.
It seems that the only obstacle to a fully-decentralized training regime, where gradients can be computed locally, is the fact that to estimate the scalar reward, we need to centrally aggregate encoded information over agents in our value network by, e.g., our max-pooling.
However, since all agents share the same policy, we should be able to assume that any agent with knowledge of the others' observations can make an estimate not only of its own action, but also the others'.
This means that each agent can also produce a \emph{local value function estimate}, as a function of the subset of the agents that it can observe.
In other words, \emph{every term} in the global (vanilla) policy gradient 
\emph{and} a value function baseline can be locally estimated from each agent's perspective.

Future work should explore the extension of these points to move towards greater contextual transferability and decentralized learning in deep RL.

\subsubsection*{Acknowledgments}
This research was supported by the National Science Foundation under grant CPS-1545116 and Berkeley DeepDrive.
We also made use of the Savio computational cluster provided by the Berkeley Research Computing program at the University of California, Berkeley.
M. A. W. thanks Rowan McAllister for his reading and feedback.

\bibliography{library}
\bibliographystyle{abbrvnat}

\end{document}